\documentclass{article}

\usepackage[nonatbib,final]{ewrl_2024}

\usepackage[utf8]{inputenc} % allow utf-8 input
\usepackage[T1]{fontenc}    % use 8-bit T1 fonts
\usepackage{url}            % simple URL typesetting
\usepackage{booktabs}       % professional-quality tables
\usepackage{microtype}
\usepackage{amsfonts}
\usepackage{amsmath}
\usepackage{amssymb}
\usepackage{amsthm}
\usepackage{mathtools}
\usepackage{bbm}
\usepackage{nicefrac}       % compact symbols for 1/2, etc.
\usepackage{microtype}      % microtypography
\usepackage{xcolor}         % colors
\usepackage[inline]{enumitem}
\usepackage[numbers]{natbib}
\usepackage[hidelinks]{hyperref}       % hyperlinks
\usepackage[algoruled, lined, shortend, linesnumbered]{algorithm2e}
\usepackage{tikz}
\usepackage{subcaption}

\DeclarePairedDelimiter{\abs}{\lvert}{\rvert}

\DeclarePairedDelimiter{\norm}{\lVert}{\rVert}

\DeclareMathOperator{\Expected}{\mathbb{E}}

\DeclareMathOperator{\Unif}{unif}
\newcommand{\Reals}{\mathbb{R}}
\newcommand{\Naturals}{\mathbb{N}}

\newcommand{\Given}{\mid}

\newcommand{\Indicator}{\mathbb{I}}
\DeclarePairedDelimiterX{\DivergencePar}[2]{(}{)}{%
  #1\;\delimsize\|\;#2%
}

\renewcommand{\Pr}{\mathbb{P}}

\newcommand{\Est}[1]{\hat{#1}}

\newcommand{\Set}[1]{\mathcal{#1}}
\newcommand{\Distribution}[1]{\Delta(#1)}
\newcommand{\Tuple}[1]{\langle #1 \rangle}
\newcommand{\Model}[1]{\mathbf{#1}}          % tuples names
\newcommand{\Nonmark}[1]{\bar{#1}}           % non-Markovian functions or processes
\newcommand{\Behav}[1]{#1^\mathsf{b}}

\newcommand{\Feature}[1]{^{(#1)}}

\newcommand{\Actions}{\Set{A}}
\newcommand{\Observations}{\Set{O}}
\newcommand{\Rewards}{\Set{R}}
\newcommand{\Histories}{\Set{H}}
\newcommand{\Episodes}{\Set{E}}

\newcommand{\Qs}{\Set{Q}}            % automaton states
\newcommand{\Inputs}{\Sigma}         % automaton inputs
\newcommand{\Outputs}{\Omega}        % automaton outputs
\newcommand{\Dataset}{\Set{D}}
\newcommand{\Alphabet}{\Gamma}
\newcommand{\Language}{{X}}
\newcommand{\Languages}{\Set{X}}

\newcommand{\RDP}{\Model{R}}

\newcommand{\Process}{\Model{P}}

\newcommand{\Policy}{\pi}

\newcommand{\TransitionFn}{\tau}                % automaton functions
\newcommand{\OutputFn}{\theta}                  % automaton functions
\newcommand{\OutputFnR}{\OutputFn_\mathsf{r}}   % reward output function
\newcommand{\OutputFnO}{\OutputFn_\mathsf{o}}   % observation output function
\newcommand{\PrefL}{L_\infty^\mathsf{p}}        % prefix L inf distance
\newcommand{\PrefLOne}{L_1^\mathsf{p}}          % prefix L 1 distance
\newcommand{\StopSymbol}{o_\bot}                   % string or episode stop symbol
\newcommand{\Boollang}{\text{B}}

\newcommand{\COccupancy}{d}                     % state or state-action occupancy distribution (it was \nu before)
\newcommand{\CDistinguish}{\mu_0}               % \PrefL-distinguishability coefficient

\newcommand{\Multiset}{\mathcal{Z}}

\newcommand{\cX}{\mathcal{X}}

\newcommand{\cZ}{\mathcal{Z}}

\newcommand{\pa}[1]{\left(#1\right)}

\newcommand{\tuple}[1]{\langle #1 \rangle}

\newcommand{\emptystring}{\lambda}

\theoremstyle{plain}
\newtheorem{theorem}{Theorem}
\newtheorem{lemma}[theorem]{Lemma}

\theoremstyle{definition}
\newtheorem{definition}{Definition}
\newtheorem{example}{Example}
\newtheorem{assumption}{Assumption}
\theoremstyle{remark}

\DontPrintSemicolon
\SetDataSty{textit}
\SetFuncSty{textsc}
\SetProcNameSty{textsc}
\SetProcArgSty{textit}

\SetCommentSty{mycommentstyle}
\SetKw{Break}{break}
\SetKwInput{KArg}{Argument}
\SetKwInput{KRet}{Returns}
\SetKwFunction{FDataSplit}{DatasetSplit}
\SetKwFunction{FAdaCT}{AdaCT--H}
\SetKwFunction{FLearnPDFA}{LearnPDFA}
\SetKwFunction{FTransformData}{TransformData}
\SetKwFunction{FOfflineRL}{OfflineRL}
\SetKwFunction{FMakeCandidates}{MakeCandidates}
\SetKwFunction{FFillMultisets}{FillMultisets}
\SetKwFunction{FTestDistinct}{TestDistinct}
\SetKwData{DSimilar}{Similar}

\title{Tractable Offline Learning of Regular Decision Processes}

\author{
  Ahana Deb \\
  Universitat Pompeu Fabra \\
  \texttt{ahana.deb@upf.edu} \\
  \And
  Roberto Cipollone \\
  Sapienza University of Rome \\
  \texttt{cipollone@diag.uniroma1.it} \\
  \And
  Anders Jonsson \\
  Universitat Pompeu Fabra \\
  \texttt{anders.jonsson@upf.edu} \\
  \And
  Alessandro Ronca \\
  University of Oxford \\
  \texttt{alessandro.ronca@cs.ox.ac.uk} \\
  \And
  Mohammad Sadegh Talebi \\
  University of Copenhagen \\
  \texttt{m.shahi@di.ku.dk} \\
}

\begin{document}

\maketitle

\begin{abstract}
This work studies offline Reinforcement Learning (RL) in a class of non-Markovian environments
called Regular Decision Processes (RDPs).
In RDPs, the unknown dependency of future observations and rewards from the past interactions can be captured by some hidden finite-state automaton.
For this reason, many RDP algorithms first reconstruct this unknown dependency using automata learning techniques.
In this paper, we show that it is possible to overcome two strong limitations of previous offline RL algorithms for RDPs, notably RegORL \citep{cipollone2023}.
This can be accomplished via the introduction of two original techniques:
the development of a new pseudometric based on formal languages, which removes a problematic dependency on $\PrefL$-distinguishability parameters, and the adoption of Count-Min-Sketch (CMS), instead of naive counting.
The former reduces the number of samples required in environments that are characterized by a low complexity in language-theoretic terms.
The latter alleviates the memory requirements for long planning horizons.
We derive the PAC sample complexity bounds associated to each of these techniques, and we validate the approach experimentally.
\end{abstract}

\section{Introduction}

% Why non-Markovian RL
The Markov assumption is fundamental for most Reinforcement Learning (RL) algorithms, requiring that the immediate reward and transition only depend on the last observation and action.
Thanks to this property, the computation of (near-)optimal policies involves only functions over observations and actions.
However, in complex environments, 
observations may not be complete representations of the internal environment state.
In this work, we consider RL in Non-Markovian Decision Processes (NMDPs).
In these very expressive models, the probability of future observations and rewards may depend on the entire history, which is the past interaction sequence composed of observations and actions.
%
% Why RDPs
The unrestricted dynamics of the NMDP formulation is not tractable for optimization.
Therefore, previous work in non-Markovian RL focus on tractable subclasses of decision processes.
In this work, we focus on Regular Decision Processes (RDPs) \citep{brafman_2019_RegularDecision,brafman2024regular}.
In RDPs, the distribution of the next observation and reward is allowed to vary according to regular properties evaluated on the history.
Thus, these dependencies can be captured by a Deterministic Finite-state Automaton (DFA).
RDPs are expressive models that can represent complex dependencies, which may be based on events that occurred arbitrarily back in time.
For example, we could model that an agent may only enter a restricted area if it has previously asked for permission and the access was granted.

% The problem
Due to the properties above, RL algorithms for RDPs are very general and applicable.
However, provably correct and sample efficient algorithms for RDPs are still missing.
On one hand, local optimization approaches are generally more efficient, but lack correctness guarantees.
In this group, we recall \citet{abadi_2020_LearningSolving,toroicarte2019learning} and all RL algorithms with policy networks that can operate on sequences.
On the other hand, algorithms with formal guarantees do not provide a practical implementation \citep{cipollone2023,ronca_2021_EfficientPACa} or can only be applied effectively in small environments \citep{ronca_2022_MarkovAbstractions}.
%
% What we do
In this work, we propose a new offline RL algorithm for RDPs with a Probably Approximately Correct (PAC) sample complexity guarantee. 
Our algorithm improves over previous work in three ways.
First, we overcome a limitation of previous sample complexity bounds,
which is the dependence on a parameter called $\PrefL$-distinguishability.
This is desirable because there exist simple non-Markovian domains in which this parameter decays exponentially with respect to the number of RDP states; an example is the T-maze by \citet{bakker2001rlandlstm}, discussed below.
Second, a careful treatment of the individual features that compose the trace and each observation allows us to further improve the efficiency.
Third, inspired by the automaton-learning algorithm FlexFringe~\citep{baum2023ff}, we use a data structure called Count-Min-Sketch (CMS) \citep{cormodeM05cms} to compactly represent probability distributions on large sets.

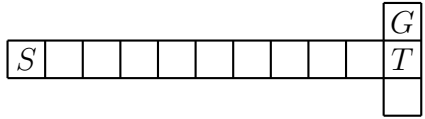
\begin{figure}
  \centering
  \begin{tikzpicture}[scale=0.5]

% Draw the horizontal corridor
\draw[thick] (0,0) -- (10,0); % bottom wall
\draw[thick] (0,1) -- (10,1); % top wall
\draw[thick] (0,0) -- (0,1); % left end wall
\draw[thick] (10,0) -- (10,1); % right end wall

% Draw the vertical branches
\draw[thick] (10,1) -- (11,1); % top horizontal
\draw[thick] (10,0) -- (11,0); % bottom horizontal
\draw[thick] (10,1) -- (10,2); % right vertical
\draw[thick] (10,0) -- (10,-1); % left vertical

% Draw the left end of the T
\draw[thick] (10,2) -- (11,2); % right horizontal top
\draw[thick] (10,-1) -- (11,-1); % right horizontal bottom
% \draw[thick] (14,2) -- (14,-1); % right vertical connecting

% Draw the cells in the corridor
\foreach \x in {1,2,...,9} {
    \draw[thick] (\x,0) -- (\x,1); % vertical lines
}

% Draw the cells in the vertical branches
\draw[thick] (11,2) -- (11,-1); % vertical middle line in T-junction

% Labels
\node at (0.5, 0.5) {\large $S$};
\node at (10.5, 0.5) {\large $T$};
\node at (10.5, 1.5) {\large $G$};
% (roberto) I removed X because I found it confusing: it's not an observation, nor a location.

\end{tikzpicture}
  \caption{T-maze~\cite{bakker2001rlandlstm} with corridor length $N=10$.
  The observation produced at the initial position~$S$ indicates the position of the goal~$G$ at the end of the corridor for the current episode.}
  \label{fig:tmaze}
\end{figure}

\begin{example}[T-maze] \label{ex:tmaze}
  As a running example, we consider the T-Maze domain \citep{bakker2001rlandlstm}, a deterministic non-Markovian grid-world environment.
  An agent has to reach a goal~$G$ from an initial position~$S$ in a corridor of length~$N$ that terminates with a T-junction as shown in Figure~\ref{fig:tmaze}.
  The agent can move one cell at a time, taking actions $\mathit{North}$, $\mathit{South}$, $\mathit{East}$, or $\mathit{West}$.
  In each episode, the rewarding goal~$G$ can be in the cell above or below the T-junction.
  Depending on the position of the goal, the observation in state $S$ is $011$ or $110$.
  In the corridor the observation is $101$, and at the T-junction the observation is $010$.
  This means that, when crossing the corridor, the agent cannot observe the current location or the goal position.
  This yields a history-dependent dynamics that cannot be modeled with an MDP or any k-order MDP.
  As we show later, this domain can be expressed as an RDP.
  %To intuitively describe how to do this, we consider a finite-state automaton with two components, each with 13 states arranged as the T-maze itself.
  %The first component of the automaton will generate a reward in the top cell, and the second in the bottom cell.
  %The initial state of the RDP will select the initial observation and transition to the corresponding component accordingly.
\end{example}

\paragraph{Contributions}
We study offline RL for Regular Decision Processes (RDPs) and address limitations of previous algorithms.
To do so, we develop a novel language metric~$L_\Languages$, based on the theory of formal languages, and define a hierarchy of language families that allows for a fine-grained characterization of the RDP complexity.
Unlike previous algorithms, the RL algorithm we develop based on this language hierarchy does not depend on $\PrefL$-distinguishability and is exponentially more sample efficient in domains having low complexity in language-theoretic terms.
In addition, we reduce the space complexity of the algorithm and modify \textsc{RegORL} with the use of Count-Min-Sketch.
To validate our claims, we provide a theoretical analysis for both variants, when the $L_\Languages$-distinguishability or CMS is used.
Finally, we provide an experimental analysis of both approaches.

\subsection{Related work}

\paragraph{RL in RDPs}
The first online RL algorithm for RDPs is provided in \citet{abadi_2020_LearningSolving}.
%which is based on a clustering technique.
Later, \citet{ronca_2021_EfficientPACa} and \citet{ronca_2022_MarkovAbstractions} developed the first online RL algorithms with sample complexity guarantees.
The algorithm and the sample complexity bound provided in \citet{ronca_2021_EfficientPACa} adapts analogous results from automata learning literature \citep{ballepigem_2013_LearningFinitestate,balle_2013_LearningProbabilistic,balle_2014_AdaptivelyLearning,clark2004pac,palmer2007pac,ron_1998_LearnabilityUsage}.
RegORL from \citet{cipollone2023} is an RL algorithm for RDPs with sample complexity guarantees for the offline setting.
In this work, we study the same setting and improve on two significant weaknesses of RegORL, regarding the sample compexity and the space requirements.
The details of these improvements are discussed in the following sections.
Lastly, the online RL algorithm in \citet{toroicarte2019learning} can also be applied to RDPs,
but it is not proven to be sample efficient.

\paragraph{Non-Markovian RL}
Apart from RDP algorithms, one can apply RL methods to more general decision processes.
Indeed, the automaton state of an RDP can be seen as an information state, as defined in \cite{subramanian_2022_ApproximateInformationState}.
As shown in \citep{brafman_2019_RegularDecision},
any RDP can also be expressed as a POMDP whose hidden dynamics evolves according to its finite-state automaton.
Therefore, any RL algorithm for generic POMDPs can also be applied to RDPs.
Unfortunately, planning and learning in POMDPs is intractable \citep{krishnamurthy_2016_PACReinforcementLearning,papadimitriou1987complexity}.
More efficient learning algorithms, with more favorable complexity bounds, have been obtained for subclasses of POMDPs, such as undercomplete POMDPs \citep{guo_2022_ProvablyEfficient,jin2020undercomplete},
few-step reachability \citep{guo2016pac},
ergodicity \citep{azizzadenesheli2016reinforcement},
few-step decodability \citep{efroni_2022_ProvableReinforcementLearning,krishnamurthy_2016_PACReinforcementLearning},
or weakly-revealing \citep{liu_2022_WhenPartiallyObservable}.
However, none of these assumptions exhaustively capture the entire RDP class, and they cannot be applied in general RDPs.

Regarding more general non-Markovian dynamics,
Predictive State Representations (PSRs) \citep{bowling2006learning,james2004learning,kulesza2015spectral,singh2003learning} are general descriptions of dynamical systems that capture POMDPs and therefore RDPs.
There exist polynomial PAC bounds for online RL in PSRs \citep{zhan2023pac}.
However, these bounds involve parameters that are specific to PSRs and do not immediately apply to RDPs.
Feature MDPs and state representations both share the idea of having a map from histories to a state space.
This is analogous to the map determined by the transition function of the automaton underlying an RDP.
Algorithmic solutions for feature MDPs are based on suffix trees, and they cannot yield optimal performance in our setting \citep{hutter2009feature,veness2011aixi}.
The automaton of an RDP can be seen as providing one kind of state representation \citep{mahmud2010constructing,maillard2011selecting,maillard2013optimal,nguyen2013competing,ortner2019regret}.
The existing bounds for state representations show a linear dependency on the number of candidate representations, which is exponential in the number of states in our case.
A similar dependency is also observed in \cite{lattimore2013sample}.
With respect to temporal dependencies for rewards, Reward Machines consider RL with non-Markovian rewards
\citep{bourel2023exploration,degiacomo2019bolts,degiacomo2020monitoring,hasanbeig2021deepsynth,toroicarte2018reward,xu2020jirp}.
However, this dependency is usually assumed to be known, which makes the Markovian state computable.
Lastly, non-Markovianity is also introduced by the logical specifications that the agent is
required to satisfy
\citep{9196796,fu2014paccontrol,DBLP:conf/tacas/HahnPSSTW19,hammond2021multiagent,hasanbeig2020cautious};
however, it is resolved a priori from the known specification.

\paragraph{Offline RL in MDPs.}
There is a rich and growing literature on offline RL, and provably sample efficient algorithms have been proposed for various settings of MDPs \citep{chen2019information,jin2021pessimism,li_2022_SettlingSample,rashidinejad2021bridging,ren2021nearly,uehara2022pessimistic,uehara2022representation,xie_2021_PolicyFinetuning,yin2021towards,zhan2022offline}. For example, in the case of episodic MDPs, it is established that the optimal sample size in offline RL depends on the size of state-space, episode length, as well as some notion of concentrability, reflecting the distribution mismatch between the behavior and optimal policies. 
A closely related problem is off-policy learning \citep{kallus2020double,maei2010toward,thomas2016data,uehara2022review}.

\section{Preliminaries}

\paragraph{Notation}
Given a set $\Set{Y}$, $\Distribution{\Set{Y}}$ denotes the set of probability distributions over~$\Set{Y}$. 
For a function ${f: \Set{X} \to \Distribution{\Set{Y}}}$,
$f(y \Given x)$ is the probability of $y\in \Set{Y}$ given~$x\in \Set{X}$. 
Further, we write $y \sim f(x)$ to abbreviate $y \sim f(\cdot \Given x)$. 
%For $y\in\Set{Y}$, we use $\Kronecker{y} \in \Distribution{\Set{Y}}$ to denote the Kronecker delta defined as $\Kronecker{y}(y)=1$ and $\Kronecker{y}(y')=0$ for each $y'\in\Set{Y}$ such that $y'\neq y$.
Given an event $E$,  $\Indicator(E)$ denotes the indicator function of $E$, which equals $1$ if $E$ is true, and $0$ otherwise. %, e.g.~$\Kronecker{y}(y')=\Indicator(y=y')$. 
For any pair of integers $m$ and $n$ such that $0\leq m\leq n$, we let $[m,n]:=\{m,\ldots,n\}$ and $[n]:=\{1, \dots, n\}$.
%Given a set $\Set{X}$, for $k \in \Naturals$, $\Set{X}^k$ represents the set of sequences of length $k$ whose elements are from $\Set{X}$. Also, $\Set{X}^* = \cup_{k=0}^\infty \Set{X}^k$.
The notation $\widetilde{\mathcal{O}}(\cdot)$ hides poly-logarithmic terms. 

\paragraph{Count-Min-Sketch}
Count-Min-Sketch, or CMS~\cite{cormodeM05cms}, is a data structure that compactly represents a large non-negative vector $v=[v_1,\ldots,v_m]$.
CMS takes two parameters $\delta_c$ and $\varepsilon$ as input, and constructs a matrix $C$ with $d=\lceil \log \frac 1 {\delta_c} \rceil$ rows and $w=\lceil \frac e \varepsilon \rceil$ columns.
For each row $j\in[d]$, CMS picks a hash function $h_j:[m]\to[w]$ uniformly at random from a pairwise independent family \citep{motwani1995randomized}.
Initially, all elements of $v$ and $C$ equal $0$.
An update $(i,c)$ consists in incrementing the element $v_i$ by $c>0$.
CMS approximates an update $(i,c)$ by incrementing $C(j,h_j(i))$ by $c$ for each $j\in[d]$.
At any moment, a point query $\widetilde{v}_i$ returns an estimate of $v_i$ by taking the minimum of the row estimates, i.e.~$\widetilde{v}_i = \min_j C(j,h_j(i))$.
It is easy to see that $\widetilde{v}_i\geq v_i$, i.e.~CMS never underestimates $v_i$.

\subsection{Languages and operators} \label{sec:languages-operators}
An \emph{alphabet} $\Alphabet$ is a finite non-empty set of elements called \emph{letters}.
A \emph{string} over $\Alphabet$ is a concatenation $a_1 \cdots a_\ell$ of letters from $\Alphabet$, and we call $\ell$ its length.
In particular, the string containing no letters, having length zero, is a valid string called the \emph{empty string}, and is denoted by $\emptystring$.
Given two strings $x_1 = a_1 \cdots a_\ell$ and $x_2 = b_1 \cdots b_m$, their concatenation $x_1x_2$ is the string $a_1 \cdots a_\ell b_1 \cdots b_m$. In particular,
for any given string $x$, we have $x \emptystring = \emptystring x = x$.
The set of all strings over alphabet $\Alphabet$ is written as $\Alphabet^*$, and the set of all strings of length $\ell$ is written as $\Alphabet^\ell$. 
Thus, $\Alphabet^* = \cup_{\ell \in \Naturals} \Alphabet^\ell$.
A \emph{language} is a subset of $\Alphabet^*$.
Given two languages $\Language_1$ and $\Language_2$, their concatenation is the language defined by 
$\Language_1 \Language_2 = \{ x_1 x_2 \mid x_1 \in \Language_1 \text{, } x_2 \in \Language_2 \}$. 
When concatenating with a language $\{ a \}$ containing a single string consisting of a letter $a \in \Alphabet$, we simply write $\Language a$ instead of $\Language \{ a \}$.
Concatenation is associative and hence we can write the concatenation $\Language_1 \Language_2 \cdots \Language_k$ of an arbitrary number of languages.
%A language~$\Language$ is a \emph{marked product} of the languages $\Language_0, \dots, \Language_n$ if
%$\Language$ is given by the concatenation $\Language_0 a_0 \Language_1 a_1 \dots \Language_n$ for some letters $a_0, \dots, a_{n-1} \in \Alphabet$.
%Given a set of languages~$\Languages$, the operator $\Polylang(\Languages)$, called polynomial closure of~$\Languages$,
%denotes the set of all marked products of languages in~$\Languages$.
%In addition, $\Polylang_k(\Languages)$ is the set of marked products of $k$ languages in~$\Languages$,
%and $\Polylang_{\le k}(\Languages) = \cup_{i=0,\dots,k}\Polylang_i(\Languages)$.
%% This is not relevant anymore. Maybe we want to make this connection later in the paper. We can skip this for now
%%The Boolean closure operator $\Boollang(\Languages)$ is the set of languages that result from repeated intersection, union, and negation of languages in~$\Languages$.
%%Together, these two operators induce a hierarchy of languages of increasing complexity, generally called the dot-depth hierarchy \citep{pin2017dot}.

Given the fundamental definitions above, we introduce two operators to construct sets of languages.
The first operator $\operatorname{C}_k^\ell$ is defined for any two non-negative integers $\ell$ and $k \in \{1,\dots,\ell\}$, it takes  a set of languages $\mathcal{G}$, and constructs a new set of languages as follows:
\begin{align*}
\operatorname{C}_k^\ell(\mathcal{G}) = 
\{  \Alphabet^\ell \cap S_1 G_1 S_2 G_2 \cdots S_k G_k S_{k+1} \mid G_1, \dots, G_k \in \mathcal{G}, \; S_1, \dots, S_{k+1} \in \mathcal{S} \},
\end{align*}
where the set $\mathcal{S} = \{ \Alphabet^*, \{ \emptystring \} \}$ consists of the language $\Alphabet^*$ of all strings over the considered alphabet and the singleton language $\{ \emptystring \}$ consisting of the empty string only.
In the definition, each $S_i$ can be $\Gamma^*$ or $\emptystring$.
Choosing $S_i = \Gamma^*$ allows arbitrary letters between a string from $G_i$ and the next string from $G_{i+1}$, whereas choosing $S_i = \{ \emptystring \}$ enforces that a string from $G_i$ is immediately followed by a string from $G_{i+1}$.
% Finally, note that $\operatorname{C}_\ell^\ell(\Alphabet) = \Alphabet^\ell$.
Also, the intersection with $\Gamma^\ell$ amounts to restricting to strings of length $\ell$ in the languages given by the concatenation.

The second operator we define, $\Boollang$, takes a set of languages $\Languages$, and it constructs a new set of languages $\Boollang(\Languages) = \Languages^{\sqcup} \cup \Languages^{\sqcap}$ by taking combinations as follows:
\begin{align*}
    % \Languages^{\mathsf{c},\ell} & = \{ \Gamma^\ell \setminus \Language \mid \Language \in \Languages \},
    % \;\;
    & \Languages^{\sqcup}  = \{ \Language_1 \cup \Language_2 \mid \Language_1,\Language_2 \in \Languages \},
    && \Languages^{\sqcap}  = \{ \Language_1 \cap \Language_2 \mid \Language_1, \Language_2 \in \Languages \}.
\end{align*}
The set $\Languages^{\sqcup}$ consists of the pair-wise unions of the languages in $\Languages$.
For example, when $\Languages = \{ \{ ac, ad \}, \{ac,bc\} \}$, we have $\Languages^{\sqcup} = \Languages \cup \{ \{ ac, ad, bc \}  \}$. 
Similarly, the set $\Languages^{\sqcap}$ consists of the pair-wise intersections of the languages in $\Languages$.
For example, when $\Languages = \{ \{ ac, ad \}, \{ac,bc\} \}$, we have $\Languages^{\sqcap} = \Languages \cup \{ \{ ac \}  \}$. 
From the previous example we can observe that $\Languages \subseteq \Languages^{\sqcup}, \Languages^{\sqcap}$, which holds in general since 
$\Language = \{ \Language \cap \Language \mid X \in \Languages \} \subseteq \Languages^{\sqcap}$ and similarly $\Language = \{ \Language \cup \Language \mid X \in \Languages \} \subseteq \Languages^{\sqcup}$
Therefore $\Languages \subseteq \Boollang(\Languages)$.
The operators will later be used to define relevant patterns on episode traces.
They are inspired by classes of languages in the first level of the \emph{dot-depth hierarchy}, a well-known hierarchy of star-free regular languages \citep{pin2017dot,simon1972hierarchies,DBLP:journals/ita/Therien05}.
% The Boolean closure $\Boollang^*(\Languages)$ is obtained by iterating $\Boollang$.

%The following is an example of a set of languages we can build with the operators above.
%Let us consider the set of all singleton languages $\mathcal{S} = \{ \{ \gamma \} \mid \gamma \in \Gamma \}$, and the set of trivial languages $\mathcal{T} = \{ \emptyset, \Gamma^* \}$.
%For some $k$,
%we define the set of marked products $\mathcal{M} = (\mathcal{T} \mathcal{S}^\mathsf{c})^k \mathcal{T}$. Then we define the set of languages $\Languages = %\mathcal{M}^{\sqcup} \cup \mathcal{M}^{\sqcap}$.
%Note that, with $k=1$, the set $\mathcal{M} = \mathcal{T} \mathcal{S}^\mathsf{c} \mathcal{T}$ would contain both the language that checks for an occurrence of %$\gamma$ anywhere in the suffix, and the language that checks for $\gamma$ appearing ``with delay 1'', i.e., at the beginning of the string, which corresponds %to the case when we instantiate the first occurrence of $\mathcal{T}$ in $\mathcal{M}$ to the empty set, and its second occurrence to $\Gamma^*$.

\subsection{Episodic regular decision processes}
We first introduce generic episodic decision processes. 
An episodic decision process is a tuple $\Process=\tuple{\Observations,\Actions,\Rewards,\Nonmark{T},\Nonmark{R},H}$, where $\Observations$ is a finite set of observations, $\Actions$ is a finite set of actions, $\Rewards\subset[0,1]$ is a finite set of rewards, and $H\geq 1$ is a finite horizon.
We frequently consider the concatenation $\Actions\Observations$ of the sets $\Actions$ and $\Observations$.
%As is common in automata theory, we use sequences $a_mo_m\cdots a_no_n$ to denote traces of actions and observations, and concatenation $\Actions\Observations=\Actions\times\Observations$ as shorthand for the Cartesian product.
Let $\Histories_t=(\Actions\Observations)^{t+1}$ be the set of histories of length $t+1$, and let $e_{m:n}\in\Histories_{n-m}$ denote a history from time $m$ to time $n$, both included.
Each action-observation pair $ao\in\Actions\Observations$ in a history has an associated reward label $r\in\Rewards$, which we write $ao/r\in\Actions\Observations/\Rewards$.
A {\em trajectory} $e_{0:T}$ is the full history generated until (and including) time $T$.
%\sadegh{If it's not including $T$, then we must have $e_{0:T-1}$ to be consistent with earlier defs.}

We assume that a trajectory $e_{0:T}$ can be partitioned into {\em episodes} $e_{\ell:\ell+H}\in\Histories_H$ of length $H+1$,
%and that the dynamics at time $T=k(H+1)+t$, $t\in[H]$, are conditionally independent of the previous episodes and all rewards, i.e.~the dynamics only depend on $a_{k(H+1)}o_{k(H+1)}\cdots a_To_T$. For $t\in[H]$, let $\Histories_t=(\Actions\Observations)^{t+1}$ denote the relevant part of the trajectory for decision making, and let $\Histories=\cup_{t=0}^H\Histories_t$.
%
%We refer to elements in $\Histories$ as {\em histories}, even though they are not complete trajectories.
In each episode $e_{0:H}$, $a_0=a_\bot$ is a dummy action used to initialize the distribution on $\Histories_0$.
The transition function $\Nonmark{T}:\Histories\times\Actions\to\Distribution{\Observations}$ and the reward function $\Nonmark{R}:\Histories\times\Actions\to\Distribution{\Rewards}$ depend on the current history in $\Histories=\cup_{t=0}^H\Histories_t$.
Given $\Process$, a generic policy is a function $\Policy:(\Actions\Observations)^*\to\Distribution{\Actions}$ that maps trajectories to distributions over actions.
The value function $V^\Policy:[0,H]\times \Histories\to\Reals$ of a policy $\Policy$ is a mapping that assigns real values to histories.
For $h\in \Histories$, it is defined as $V^\Policy(H, h)\coloneqq 0$ and
\begin{equation}
	V^\Policy(t, h) \coloneqq \Expected \left[\, \sum_{i=t+1}^H r_i \, \middle\vert \, h, \Policy \right], \;\; \forall t<H, \forall h\in \Histories_t.
	\label{eq:value-as-sum}
\end{equation}
For brevity, we write $V^\pi_t(h):=V^\pi(t,h)$.
The optimal value function $V^*$ is defined as $V^*_t(h) \coloneqq \sup_{\pi} V^\pi_t(h), \forall t\in [H], \forall h\in \Histories_t$,  
where $\sup$ is taken over all policies $\Policy:(\Actions\Observations)^*\to\Distribution{\Actions}$. Any policy achieving  $V^*$ is called optimal, which we denote by $\pi^*$; namely $V^{\pi^*}=V^*$. 
Solving $\Process$ amounts to finding $\Policy^*$.
In what follows, we consider simpler policies of the form $\Policy:\Histories\to\Distribution{\Actions}$ mapping finite histories to distributions over actions.
Let $\Pi_\Histories$ denote the set of such policies.
It can be shown that $\Pi_\Histories$ always contains an optimal policy, i.e.~$V^*_t(h) \coloneqq \max_{\pi\in\Pi_\Histories} V^\pi_t(h), \forall t\in [H], \forall h\in \Histories_t$. An episodic MDP is an episodic decision process whose dynamics at each timestep~$t$ only depends on the last observation and action \citep{puterman_1994_MarkovDecision}.

\paragraph{Episodic RDPs}
An episodic Regular Decision Process (RDP)~\citep{abadi_2020_LearningSolving,brafman_2019_RegularDecision,brafman2024regular} is an episodic decision process $\RDP=\tuple{\Observations,\Actions,\Rewards,\Nonmark{T},\Nonmark{R},H}$ described by a \emph{finite transducer} (Moore machine)  $\Tuple{ \Qs, \Inputs, \Outputs, \TransitionFn, \OutputFn, q_0}$, where $\Qs$ is a finite set of states, $\Inputs= \Actions\, \Observations$ is a finite input alphabet composed of actions and observations, $\Outputs$ is a finite output alphabet, $\TransitionFn: \Qs \times \Inputs \to \Qs$ is a transition function, $\OutputFn: \Qs \to \Outputs$ is an output function, and $q_0\in\Qs$ is a fixed initial state \citep{moore1956gedanken,shallit2008second}. 
The output space $\Outputs = \Outputs_\mathsf{o} \times \Outputs_\mathsf{r}$ consists of a finite set of functions that compute the conditional probabilities of observations and rewards, on the form $\Outputs_\mathsf{o} \subset \Actions \to \Distribution{\Observations}$ and $\Outputs_\mathsf{r} \subset \Actions \to \Distribution{\Rewards}$.
For simplicity, we use two output functions, $\OutputFnO: \Qs \times \Actions \to \Distribution\Observations$ and $\OutputFnR: \Qs \times \Actions \to \Distribution\Rewards$, to denote the individual conditional probabilities.
Let $\TransitionFn^{-1}$ denote the inverse of $\TransitionFn$, i.e.~$\TransitionFn^{-1}(q)\subseteq\Qs\times\Actions\Observations$ is the subset of state-symbol pairs that map to $q\in\Qs$.
%In this context, an input symbol is an element of $\Actions \Observations$.
An RDP $\RDP$ implicitly represents a function $\bar\TransitionFn: \Histories \to \Qs$ from histories in $\Histories$ to states in $\Qs$, recursively defined as
$\bar\TransitionFn(h_0) \coloneqq \TransitionFn(q_0, a_0 o_0)$, where $a_0$ is some fixed starting action, and $\bar\TransitionFn(h_t) \coloneqq \TransitionFn(\bar\TransitionFn(h_{t-1}), a_t o_t)$.
The dynamics and of $\RDP$ are defined as $\Nonmark{T}(o \Given h,a)=\OutputFnO(o \Given \bar\TransitionFn(h),a)$ and $\Nonmark{R}(r \Given h,a)=\OutputFnR(r \Given \bar\TransitionFn(h),a)$, $\forall h\in\Histories, \forall ao/r\in\Actions\Observations/\Rewards$.
As in previous work~\citep{cipollone2023}, we assume that any episodic RDP generates a designated termination observation~$\StopSymbol \in \Observations$ after exactly~$H$ transitions.
This ensures that any episodic RDP is acyclic, i.e.~the states can be partitioned as $\Qs=\Qs_0\cup\cdots\cup\Qs_{H+1}$, where each $\Qs_{t+1}$ is the set of states generated by the histories in $\Histories_t$ for each $t\in[0,H]$.
An RDP is minimal if its Moore machine is minimal.
We use $A, R, O, Q$ to denote the cardinality of $\Actions, \Rewards, \Observations, \Qs$, respectively, and assume $A\geq 2$ and $O\geq 2$.

Since the conditional probabilities of observations and rewards are fully determined by the current state-action pair $(q,a)$, an RDP $\RDP$ adheres to the Markov property over its states, but \emph{not over the observations}.
Given a state $q_t \in \Qs$ and an action $a_t \in \Actions$,
the probability of the next transition is
\[
	\Pr(r_t, o_t, q_{t+1} \Given q_t, a_t, \RDP) =
		\OutputFnR(r_t \Given q_t, a_t)\, \OutputFnO(o_t \Given q_t, a_t)\, \Indicator(q_{t+1}=\TransitionFn(q_t,a_t o_t)).
\]
Since RDPs are Markovian in the unobservable states~$\Qs$, there is an important class of policies that is called \emph{regular}.
Given an RDP $\RDP$, a policy $\Policy:\Histories\to\Distribution{\Actions}$ is called \emph{regular} if $\Policy(h_1)=\Policy(h_2)$ whenever $\bar\TransitionFn(h_1)=\bar\TransitionFn(h_2)$, for all $h_1,h_2\in \Histories$.
Hence, we can compactly define a regular policy as a function of the RDP state, i.e.~$\pi:\Qs\to\Distribution{\Actions}$.
Let $\Pi_\RDP$ denote the set of regular policies for $\RDP$.
Regular policies exhibit powerful properties. First, under a regular policy, suffixes have the same probability of being generated for histories that map to the same RDP state. Second, there exists at least one optimal policy that is regular.
Finally, in the special case where an RDP is Markovian in both observations and rewards, it reduces to a nonstationary episodic MDP.

\begin{example}[RDP for T-maze]\label{ex:tmaze-rdp}
  Consider the T-maze described in Example~\ref{ex:tmaze}.
  This can be modeled as an episodic RDP 
  $\Tuple{ \Qs, \Inputs, \Outputs, \TransitionFn, \OutputFn, q_0}$
  with states $\Qs \coloneqq q_0 \cup (\{ q_{\top 1}, \dots, q_{\top 13}\} \cup \{ q_{\bot 1}, \dots, q_{\bot 13}\}) \times \{1, \dots, H\}$, which include the initial state~$q_0$ and two parallel components, $\top$ and $\bot$, for the 13 cells of the grid world.
  In addition, each state also includes a counter for the time step.
  Within each component $\{(q_{\top i}, t)\}_i$ and $\{(q_{\bot i}, t)\}$, the transition function~$\TransitionFn$ mimics the grid world dynamics of the maze and increments the counter~$t$.
  From the initial state and the start action~$a_0$, $\TransitionFn(q_0, a_0 o_0)$ equals $q_{1,\top}$ if $o_0 = 110$, and $q_{1,\bot}$ if $o_0 = 011$.
  Observations are deterministic as described in Example~\ref{ex:tmaze},
  except for $\OutputFnO(q_0, a_0) = \Unif\{110,011\}$.
  The rewards are null, except for a 1 in the top right or bottom right cell, depending if the current state is in the component $\top$ or $\bot$, respectively.
\end{example}

\paragraph{Occupancy and distinguishability}
Given a regular policy $\Policy:\Qs\to\Distribution{\Actions}$ and a time step $t\in[0,H]$, let $\COccupancy_t^\Policy\in\Distribution{\Qs_t\times\Actions\,\Observations}$ be the induced {\em occupancy}, i.e.~a probability distribution over the states in $\Qs_t$ and the input symbols in $\Actions\,\Observations$, recursively defined as $\COccupancy_0^\Policy(q_0,a_0o_0) = \OutputFnO(o_0 \Given q_0, a_0)$ and 
\begin{align*}
%\COccupancy_0^\Policy(q_0,a_0o_0) &= \OutputFnO(q_0, a_0, o_0),\\
\COccupancy_t^\Policy(q_t,a_to_t) &= \sum_{(q,ao)\in\TransitionFn^{-1}(q_t)} \COccupancy_{t-1}^\Policy(q,ao) \cdot \Policy(a_t \Given q_t) \cdot \OutputFnO(o_t \Given q_t, a_t), \;\; t>0.
\end{align*}
Of particular interest is the occupancy distribution $\COccupancy_t^* \coloneqq \COccupancy_t^{\Policy^*}$ associated with an optimal policy $\Policy^*$. Let us assume that $\Policy^*$ is unique, which further implies that $d^*_t$ is uniquely defined\footnote{This assumption is imposed to ease the definition of the concentrability coefficient that follows, and is also considered in the offline  reinforcement learning literature (e.g., \cite{nguyen2023instance}). In the general case, one may consider the set $\Set{D}^*$ of occupancy distributions, collecting occupancy distributions of all optimal policies.}.

Consider a minimal RDP~$\RDP$ with states $\Qs = \cup_{t \in [0,H+1]} \Qs_t$.
Given a regular policy $\Policy\in\Pi_\RDP$ and a time step $t\in[0,H]$, each RDP state $q \in \Qs_t$ defines a unique probability distribution $\Pr(\cdot\Given q_t = q, \Policy)$ on episode suffixes in $\Episodes_{H-t}=(\Actions \Observations / \Rewards )^{H-t+1}$.
The states in $\Qs_t$ can be compared in terms of the probability distributions they induce over $\Episodes_{H-t}$.
Consider any $L = \{L_\ell\}_{\ell = 0}^{H}$, where each $L_\ell$ is a metric over $\Distribution{\Episodes_\ell}$.
We define the \emph{$L$-distinguishability} of $\RDP$ and $\Policy$ as the maximum $\CDistinguish\geq 0$ such that,
for any $t \in [0,H]$ and any two distinct $q, q' \in \Qs_t$, the probability distributions over suffix traces $e_{t:H} \in \Episodes_\ell$ from the two states satisfy
\[
L_{H-t}(\Pr(e_{t:H} \Given q_t = q, \Policy), \Pr(e_{t:H} \Given q_t = q', \Policy)) \ge \CDistinguish\, .
\]
We will often omit the remaining episode length $\ell = H-t$ from $L_\ell$ and simply write~$L$. 
%Instantiating the definition above,we consider the $\PrefL$-distinguishability as the parameter~$\CDistinguish$ defined according to the metric $\PrefL(p_1, p_2) = \max_{u \in [\ell], e \in \Episodes_u} \abs{p_1(e\, *) - p_2(e\, *)}$, where $p_i(e\,*)$ represents the probability of the trace prefix~$e\in\Episodes_u$, followed by any trace~$e' \in \Episodes_{\ell-u-1}$.
We consider the $\PrefL$-distinguishability, instantiating the definition above with  
the metric $\PrefL(p_1, p_2) = \max_{u \in [0,\ell], e \in \Episodes_u} \abs{p_1(e\, *) - p_2(e\, *)}$,
where $p_i(e\,*)$ represents the probability of the trace prefix~$e\in\Episodes_u$, followed by any trace~$e' \in \Episodes_{\ell-u-1}$. The $\PrefLOne$-distinguishability is defined analogously using $\PrefLOne(p_1, p_2) = \sum_{u \in [0,\ell], e \in \Episodes_u} \abs{p_1(e\, *) - p_2(e\, *)}$.

\section{Learning RDPs with state-merging algorithms from offline data}

Here we describe an algorithm called \textsc{AdaCT-H}~\cite{cipollone2023} for learning episodic RDPs from a dataset of episodes.
The algorithm starts with a set $\Dataset$ of episodes generated using a regular behavior policy $\Behav\Policy$, where the $k$-th episode is of the form 
$e_{0:H}^k=a_0^ko_0^k/r_0^k\cdots a_H^ko_H^k/r_H^k$ and where, for each $t\in[0,H]$,
$$
q_0^k = q_0, \quad a_t^k\sim \Behav\Policy(q_t^k), \quad o_t^k\sim \OutputFnO(q_t^k,a_t^k), \quad r_t^k\sim \OutputFnR(q_t^k,a_t^k), \quad q_{t+1}^k = \TransitionFn(q_t^k,a_t^ko_t^k).
$$
Note that the behavior policy $\Behav\Policy$ and underlying RDP states $q_t^k$ are unknown to the learner.
The algorithm is an instance of the PAC learning framework and takes as input an accuracy $\varepsilon\in (0,H]$ and a failure probability $\delta\in(0,1)$.
The aim is to find an $\varepsilon$-optimal policy $\widehat\Policy$ satisfying $V_0^*(h) - V_0^{\widehat{\Policy}}(h)\leq\varepsilon$ for each $h\in\Histories_0$ with probability at least $1-\delta$, using the smallest dataset $\Dataset$ possible.

Since \textsc{AdaCT-H} performs offline learning, it is necessary to control the mismatch in occupancy between the behavior policy $\Behav\Policy$ and the optimal policy $\Policy^*$. 
Concretely, the single-policy RDP concentrability coefficient associated with RDP $\RDP$ and behavior policy $\Behav\Policy$ is defined as \cite{cipollone2023}:
	\begin{equation}
		C_{\RDP}^* = \max_{t \in [H], q \in \Qs_t, ao \in \Actions\Observations}
			\frac{\COccupancy_t^{*}(q, ao)}{\Behav\COccupancy_t(q, ao)} \, .
	\end{equation}
As \citep{cipollone2023}, we also assume that the concentrability is bounded away from infinity, i.e.~that $C_{\RDP}^*<\infty$.

\textsc{AdaCT-H} is a state-merging algorithm that iteratively constructs the set of RDP states $\Qs_1,\ldots,\Qs_H$ and the associated transition function~$\TransitionFn$.
For each $t\in[0,H]$, \textsc{AdaCT-H} maintains a set of candidate states $qao\in\Qs_{t-1}\times\Actions\Observations$.
Each candidate state $qao$ has an associated multiset of suffixes $\Multiset(qao)=\{e_{t:H}^k : e^k\in\Dataset, \bar\TransitionFn(e_{0:t-1}^k)=q, a_{t-1}^ko_{t-1}^k=ao\}$, i.e.~episode suffixes whose history is consistent with $qao$.
To determine whether or not the candidate $qao$ should be promoted to $\Qs_t$ or merged with an existing RDP state $q_t$, 
\textsc{AdaCT-H} compares the empirical probability distributions on suffixes using the prefix distance $\PrefL$ defined earlier.
For reference, we include the pseudocode of \textsc{AdaCT-H}$(\Dataset,\delta)$ in Appendix~\ref{app:adacth}.
\citet{cipollone2023} prove that \textsc{AdaCT-H}$(\Dataset,\delta)$ constructs a minimal RDP $\RDP$ with probability at least $1-4AOQ\delta$.

In practice, the main bottleneck of \textsc{AdaCT-H} is the statistical test on the last line,
\[
\PrefL(\Multiset_1,\Multiset_2)\geq\sqrt{2\log(8(ARO)^{H-t}/\delta)/\min(|\Multiset_1|,|\Multiset_2|)},
\]
since the number of episode suffixes in $\Episodes_\ell$ is exponential in the current horizon $\ell$.
The purpose of the present paper is to develop tractable methods for implementing the statistical test.
These tractable methods can be directly applied to any algorithm that performs such statistical tests, e.g.~the approximation algorithm \textsc{AdaCT-H-A}~\cite{cipollone2023}.

\section{Tractable offline learning of RDPs}

% What we are trying to solve
The lower bound derived in \citet{cipollone2023} shows that sample complexity of learning RDPs is inversely proportional to the $\PrefLOne$-distinguishability.
When testing candidate states of the unknown RDP, $\PrefLOne$ is the metric that allows maximum separation between distinct distributions over traces.
Unfortunately, accurate estimates of $\PrefLOne$ are impractical to obtain for distributions over large supports---in our case the number of episode suffixes which is exponential in the horizon.
Accurate estimates of $\PrefL$ are much more practical to obtain.
However, there are instances for which states can be well separated in the $\PrefLOne$-norm, but have an $\PrefL$-distance that is exponentially small.
To address these issues, in this section we develop two improvements over the previous learning algorithms for RDPs.

\subsection{Testing in structured languages}

\paragraph{The language of traces}
Under a regular policy, any RDP state $\Qs_t$ uniquely determines a distribution over the remaining trace $a_t o_t / r_t \dots a_H o_H / r_H \in (\Actions \Observations / \Rewards)^{H-t+1}$.
Existing work \citep{cipollone2023,ronca_2021_EfficientPACa,ronca_2022_MarkovAbstractions} treats each element $a_i o_i / r_i$ as an independent atomic symbol.
This approach neglects the internal structure of each tuple and of the observations, which are often composed of features.
As a result, common conditions such as the presence of a reward become unpractical to express.
In this work, we allow observations to be composed of internal features.
Let $\Observations = \Observations\Feature{1} \times \dots \times \Observations\Feature{m}$ be an observation space.
We choose to see it as the language 
$\Observations = \Observations\Feature{1} \cdots \Observations\Feature{m}$ given by the concatenation of the features.
Then, instead of representing an observation $(o\Feature{1}, \dots, o\Feature{m})$ as an atomic symbol, we consider it as the word 
$o\Feature{1} o\Feature{2} \cdots o\Feature{m}$.
This results into traces $\Episodes_{i} = (\Actions\Observations\Feature{1} \cdots \Observations\Feature{m}/\Rewards)^{i+1}$,
%\Todo{Note about notation in overleaf comments}
and each regular policy and RDP state uniquely determine a distribution over strings from $\Episodes_{H-t+1}$.
This fine-grained representation greatly simplifies the representation of most common conditions,
such as the presence of specific rewards or features in the observation vector.

\paragraph{Testing in the language metric}

The metrics induced by the $L_1$ and $L_\infty$ norms are completely generic and can be applied to any distribution.
Although this is generally an advantage, this means that they do not exploit the internal structure of the sample space.
In our application, a sample is a trace that, as discussed above, can be regarded as a string of a specific language.
An important improvement, proposed by \citet{ballepigem_2013_LearningFinitestate}, is to consider $\PrefLOne$ and $\PrefL$, which take into account the variable length and conditional probabilities of longer suffixes.
This was the approach followed by the previous RDP learning algorithms.
However, these two norms are strongly different and lead to dramatically different sample and space complexities.

In this section, we define a new formalism that unifies both metrics and will allow the development of new techniques for distinguishing distributions over traces.
Specifically, instead of expressing the probability of single strings, we generalize the concepts above by considering the probability of \emph{sets} of strings.
A careful selection of the sets to consider, which are languages, will allow an accurate trade-off between generality and complexity.
\begin{definition}
    Let $\ell \in \mathbb{N}$, let $\Alphabet$ be an alphabet, and let $\Languages$ be a set of languages consisting of strings in~$\Alphabet^\ell$.
    The \emph{language metric} in~$\Languages$
    is the function $L_\Languages: \Distribution{\Alphabet^\ell} \times \Distribution{\Alphabet^\ell} \to \Reals$, on 
    pairs of probability distributions $p,p'$ over $\Gamma^\ell$, defined as 
    \begin{equation}
        L_\Languages(p, p') \coloneqq \max_{\Language \in \Languages} \abs{p(\Language) - p'(\Language)},
    \end{equation}
    where the probability of a language is $p(\Language) \coloneqq \sum_{x \in \Language} p(x)$.
\end{definition}
This original notion unifies all the most common metrics. Considering distributions over $\Gamma^\ell$, 
when $\Languages = \{\{x\} \mid x \in \Alphabet^\ell\}$,
the language metric $L_\Languages$ reduces to~$L_\infty$.
When $\Languages = 2^{\Alphabet^\ell}$, which is the set of all languages in~$\Alphabet^\ell$,
the language metric becomes the total variation distance, which is half the value of $L_1$.
A similar reduction can be made for the prefix distances.
The language metric captures~$\PrefL$ when $\Languages = \{x \Alphabet^{\ell-t} \mid t \in [0,\ell], x \in \Alphabet^t \}$, and it
captures~$\PrefLOne$ when $\Languages = \cup_{t \in [0,\ell]} 2^{\Alphabet^t}$.%~$\PrefLOne$.

\paragraph{Testing in language classes}

The language metric can be applied directly to the language of traces and used for testing in RDP learning algorithms.
In particular, it suffices to consider any set of languages~$\Languages$ that satisfy $\Language \subseteq \Episodes_{H-h} = (\Actions\Observations\Feature{1} \cdots \Observations\Feature{m}/\Rewards)^{H-t+1} \subseteq \Alphabet^{H-t+1}$, for each $\Language \in \Languages$.
However, as we have seen above, the selection of a specific set of languages~$\Languages$ has a dramatic impact on the metric that is being captured.
In this section, we study an appropriate trade-off between generality and sample efficiency, obtained through a suitable selection of~$\Languages$.
Intuitively, we seek to evaluate the distance between candidate states based on increasingly complex sets of languages. We present a way to construct a hierarchy of sets of languages of increasing complexity. As a first step, we define sets  of basic patterns $\mathcal{G}_i$ of increasing complexity. 
\begin{align*} 
    \mathcal{G}_1 = \ &
    \big\{ a\Observations / \Rewards \mid a \in \Actions \big\} 
    \cup  \big\{ \Actions\Observations / r \mid r \in \Rewards \big\}
    \cup 
    \big\{ \Actions\Observations\Feature{1} \cdots o\Feature{i} \cdots \Observations\Feature{m} / \Rewards \mid i \in [m] \text{, } o\Feature{i} \in \Observations\Feature{i} \big\},
    \\[2pt]
    \mathcal{G}_i = \ & \mathcal{G}_{i-1} \cup \Boollang(\mathcal{G}_{i-1}), \quad \forall i \in [2,m+2].
\end{align*}
In particular, $\mathcal{G}_1$ focuses on single components, by matching an action $a$, a reward $r$, or a single observation feature $o\Feature{i}$. At the opposite side of the spectrum, the set $\mathcal{G}_{m+2}$ considers every possible combination of actions, observation, reward; namely, it includes the set of singleton languages
$\big\{ \{ ao\Feature{1} \cdots o\Feature{m}/r \} \mid ao\Feature{1} \cdots o\Feature{m}/r \in \Actions\Observations\Feature{1} \cdots \Observations\Feature{m}/\Rewards \big\}$,
which is the most fine-grained choice of patterns, but it grows exponentially with $m$. On the contrary, the cardinality of $\mathcal{G}_1$ is linear in $m$. 
Starting from the above hiearchy, we identify two more dimensions along which complexity can grow. One dimension results from concatenating the basic patterns from $\mathcal{G}_i$, and it is obtained by applying the operator $\operatorname{C}_k^\ell$. The other dimension results from Boolean combinations of languages, and it is obtained by applying the operator $\Boollang$. 
%The two operators are defined in Section~\ref{sec:languages-operators}.
Thus, letting $\ell = {H-t+1}$, 
we define the following \emph{three-dimensional hierarchy} of sets $\Languages_{i,j,k}$ of languages:
\begin{align*}
    \Languages_{i,j,1} & = \Languages_{i,j-1,1} \cup \operatorname{C}_j^\ell(\mathcal{G}_i), && \hspace{-5pt} \forall i \in [1,m+2], \forall j \in [\ell],
    \\
    \Languages_{i,j,k} & = \Languages_{i,j,k-1} \cup \Boollang(\Languages_{i,j,k-1}), && \hspace{-5pt} \forall i \in [1,m+2], \forall j \in [\ell], \; \forall k \in [2,\ell].
\end{align*}
The family $\Languages_{i,j,k}$ induces a family of language metrics $L_{\Languages_{i,j,k}}$, which are non-decreasing along the dimensions of the hiearchy:
\begin{align*}
L_{\Languages_{i,j,k}} \leq L_{\Languages_{i+1,j,k}}, L_{\Languages_{i,j+1,k}}, L_{\Languages_{i,j,k+1}}, \quad \forall i \in [1,m+2], \forall j \in [\ell], \; \forall k \in [\ell],
\end{align*}
and it may actually be increasing since we include more and more languages as we move towards higher levels of the hierarchy.
Moreover, the last levels $\Languages_{m+2,\ell,k}$ satisfy 
$L_{\Languages_{m+2,\ell,k}} \geq \PrefL$ for every $k \in [\ell]$
since $\{x \Episodes^{\ell-t} \mid t \in [0,\ell],\, x \in \Episodes^t \} \subseteq \Languages_{m+2,\ell,k}$. Therefore, the metric $L_{\Languages_{m+2,\ell,k}}$ is at least as effective as $\PrefL$ in distinguishing candidate states. It can be much more effective as shown next.

%$\Languages_0 \subset \Languages_1 \subset \dots$.
%%
%We define the following hierarchy of languages.
%The base set contains the empty and the full languages $\Languages_0 \coloneqq \{\emptyset, \Alphabet^*\}$.
%The other sets are defined using the polynomial operator $\Polylang_i$ and the Boolean operators $\Languages^{\mathsf{c}}$, $\Languages^{\sqcup}$, and %$\Languages^{\sqcap}$.
%We group these three into a single operator $\Boollang(\Languages)$ that returns $\Languages^{\mathsf{c}} \cup \Languages^{\sqcup} \cup \Languages^{\sqcap}$.
%We can now define our hierarchy of languages as
%$\Languages_k \coloneqq \Boollang\Polylang_{k}(\Languages_0)$, for any $k \in \Naturals_+$.
%Any set $\Languages_k$ in this hierarchy induces an associated language metric $L_{\Languages_k}$.

\begin{example}[Language metric in T-maze]
  We now discuss the importance of our language metric for distinguishability,
  with respect to the T-maze described in Example~\ref{ex:tmaze} and the associated RDP in Example~\ref{ex:tmaze-rdp}.
  Consider the two states $(q_{\top 6}, 6)$ and $(q_{\bot 6}, 6)$ in the middle of the corridor and assume that the behavior policy is uniform.
  In these two states, the probability of each future sequence of actions, observations, and rewards is identical, except for the sequences that reach~$S$ or the top/bottom cells.
  These differ in the starting observation or the reward, respectively.
  Thus, the maximizer in the definition of $\PrefL$-distinguishability is the length $u=6$ and any string $x$ that contains the initial observation.
  Since the behavior policy is a random walk, the probability of~$x$ is $0.5^6$ from one state, and 0 from the other, depending on the observation.
  Therefore, $\PrefL$-distinguishability decreases exponentially with the length of the corridor.
  Consider instead the language $\Languages_{1,1,1}$ and its associated $L_{\Languages_{1,1,1}}$.
  This set of languages includes $\Alphabet^* \Actions ``110" \Rewards \Alphabet^* \cap \Alphabet^\ell$, which represents the language of any episode suffix containing the observation $110$.
  Since it does not depend on specific paths, its probability is 0 for $(q_{\bot 6}, 6)$ and it equals the probability of visiting the leftmost cell for $(q_{\top 6}, 6)$.
  For sufficiently long episodes, this probability approaches~1.
  Thus, in some domains, the language metric improves exponentially over $\PrefL$.
\end{example}

\begin{assumption}
    The behavior policy~$\Behav\Policy$ has an $L_{\Languages_{i,j,k}}$-distinguishability of at least~$\CDistinguish > 0$,
    where $\Languages_{i,j,k}$ is constructed as above and is an input to the algorithm.
\end{assumption}

Given $i\in[0,H]$, let $p, p' \in \Distribution{\Episodes_{i}}$ be two distributions over traces.
To have accurate estimates for the language metric over some~$\Languages_{i,j,k}$,
we instantiate two estimators, $\Est{p}$ and $\Est{p}'$, respectively built using datasets of episodes $\Set{C}$ and $\Set{C'}$, defined as the fraction of samples that belong to the language; 
that is, $\Est{p} \coloneqq \sum_{e \in \Set{C}} \Indicator(e \in \Languages_{i,j,k}) / |\Set{C}|$
and $\Est{p}' \coloneqq \sum_{e \in \Set{C}'} \Indicator(e \in \Languages_{i,j,k}) / |\Set{C}'|$.

\subsection{Analysis}

In this section, we show how to derive high-probability sample complexity bounds for \textsc{AdaCT-H} when Count-Min-Sketch (CMS) or the language metric $L_\Languages$ is used.
Intuitively, it is sufficient to show that the statistical test is correct with high probability; the remaining proof is identical to that of \citet{cipollone2023}.

\begin{theorem}
\label{thm:cms}
\textsc{AdaCT-H}$(\Dataset,\delta)$ returns a minimal RDP $\RDP$ with probability at least $1-4AOQ\delta$ when CMS is used to store the empirical probability distributions of episode suffixes, the statistical test is
\[
\PrefL(\Multiset_1,\Multiset_2)\geq\sqrt{8\log(16(ARO)^{H-t}/\delta)/\min(|\Multiset_1|,|\Multiset_2|)},
\]
and the size of the dataset is at least $|\Dataset|\geq\widetilde{\mathcal{O}}(\sqrt{H}/\Behav d_{\min} \mu_0)$, where $\Behav d_{\min} = \min_{t,q,ao}\Behav d_t(q,ao)$.
\end{theorem}
The proof of Theorem~\ref{thm:cms} appears in Appendix~\ref{app:thms}.

\begin{theorem}
\label{thm:lang}
\textsc{AdaCT-H}$(\Dataset,\delta)$ returns a minimal RDP $\RDP$ with probability at least $1-4AOQ\delta$ when the statistical test is implemented using the language metric $L_\Languages$ and equals
\[
L_\Languages(\Multiset_1,\Multiset_2)\geq\sqrt{2\log(4|\Languages|/\delta)/\min(|\Multiset_1|,|\Multiset_2|)},
\]
and the size of the dataset is at least $|\Dataset|\geq\widetilde{\mathcal{O}}(1/\Behav d_{\min} \mu_0)$.
\end{theorem}
The proof of Theorem~\ref{thm:lang} also appears in Appendix~\ref{app:thms}.
Note that by definition, $\mu_0$ is the $L$-distinguishability of $\RDP$ for the chosen language set $\Languages$, which has to satisfy $\mu_0>0$ for \textsc{AdaCT-H} to successfully learn a minimal RDP.
Even though our sample complexity results are similar to those of previous work up to a factor $\sqrt{H}$, as discussed earlier, $\mu_0$ may be exponentially smaller for $L_\Languages$ than for $\PrefL$.
We remark that the analysis does not hold if CMS is used to store the empirical probability distribution of the language metric $L_\Languages$, since the languages in a set $\Languages$ may overlap, causing the probabilities to sum to a value significantly greater than 1.

\section{Experimental Results}

In this section we present the results of experiments with two versions of \textsc{AdaCT-H}: one that uses CMS to compactly store probability distributions on suffixes, and one that uses the restricted language family $\Languages_{1,1,1}$.
We compare against FlexFringe~\cite{baum2023ff}, a state-of-the-art algorithm for learning probabilistic deterministic finite automata, which includes RDPs as a special case.
To approximate episodic traces, we add a termination symbol to the end of each trace, but FlexFringe sometimes learns RDPs with cycles.
Moreover, FlexFringe uses a number of different heuristics that optimize performance, but these heuristics no longer preserve the high-probability guarantees.
Hence the automata output by FlexFringe are not always directly comparable to the RDPs output by \textsc{AdaCT-H}.

We perform experiments in five domains from the literature on POMDPs and RDPs: Corridor~\citep{ronca_2021_EfficientPACa}, T-maze~\citep{bakker2001rlandlstm}, Cookie~\cite{toroicarte2019learning}, Cheese~\cite{McCallum1996ReinforcementLW} and Mini-hall~\cite{littman}.
Appendix~\ref{app:exps} contains a detailed description of each domain, as well as example automata learned by the different algorithms.
Table~\ref{table:1} summarizes the results of the three algorithms in the five domains.
We see that \textsc{AdaCT-H} with CMS is significantly slower than FlexFringe, which makes sense since FlexFringe has been optimized for performance.
CMS compactly represents the probability distributions over suffixes, but \textsc{AdaCT-H} still has to iterate over all suffixes, which is exponential in $H$ for the $\PrefL$ distance.
As a result, CMS suffers from slow running times, and exceeds the alloted time budget of 1800 seconds in the Mini-hall domain.

On the other hand, \textsc{AdaCT-H} with the restricted language family $\Languages_{1,1,1}$ is faster than FlexFringe in all domains except T-maze, and outputs smaller automata than both FlexFringe and CMS in all domains except Mini-hall.
The number of languages of $\Languages_{1,1,1}$ is linear in the number of actions, observation symbols, and reward values, so the algorithm does not have to iterate over all suffixes, and the resulting RDPs better exploit the underlying structure of the domains.
In Corridor and Cookie, all algorithms learn automata that admit an optimal policy.
However, in T-maze, Cheese and Mini-hall, the RDPs learned by \textsc{AdaCT-H} with the restricted language family  $\Languages_{1,1,1}$ admit a policy that outperforms those of FlexFringe and CMS.
As mentioned, the heuristics used by FlexFringe are not optimized to preserve reward, which is likely the reason why the derived policies perform worse.

\begin{table} % 
  \centering
  \caption{Summary of the experiments. For each domain, $H$ is the horizon, and for each algorithm, $Q$ is the number of states of the learned automaton, $r$ is the average reward of the derived policy, and time is the running time in seconds of automaton learning.}
\begin{tabular}{l@{\hspace*{9pt}}l@{\hspace*{9pt}}l@{\hspace*{9pt}}l@{\hspace*{9pt}}l@{\hspace*{9pt}}l@{\hspace*{9pt}}l@{\hspace*{9pt}}l@{\hspace*{9pt}}l@{\hspace*{9pt}}l@{\hspace*{9pt}}l}
    \toprule
    \multicolumn{1}{c}{} & \multicolumn{1}{c}{} &\multicolumn{3}{c}{FlexFringe} & \multicolumn{3}{c}{CMS} & \multicolumn{3}{c}{Restricted Languages} 
                      \\
    \cmidrule(r){3-5} \cmidrule(r){6-8} \cmidrule(r){9-11}
      Name & $H$& $Q$ & $r$ & time & $Q$ & $r$ & time & $Q$ & $r$ & time  \\
    \midrule
    Corridor     & $5$ 	  & $11$ & $1.0$ & $0.03$ & $11$ & $1.0$ &$0.3$ &  $11$ & $1.0$ & $0.01$  \\
    T-maze & $5$ & $29$ & $0.0$ & $0.11$& $104$ & $4.0$& $10.1$ & $18$ & $4.0$ & $0.26$    \\
    Cookie     &$9$ &$220$ & $1.0$ & $0.36$ & $116$ & $1.0$ & $6.05$ & $91$ & $1.0$ & $0.08$\\
    Cheese   & $6$ & $669$ & $0.69 \pm .04$ & $19.28$ & $1158$ & $0.4 \pm .05$& $207.4$ & $326$ & $0.81 \pm .04$ & $2.23$  \\
    Mini-hall   & $15$ & $897$ & $0.33 \pm .04$ & $25.79$ & - &-&- & $5134$ & $0.91 \pm .03 $ & $23.9$ \\
    \bottomrule
  \end{tabular}
  \label{table:1}
\end{table}

\section{Conclusion}
In this paper, we propose two new approaches to offline RL for Regular Decision Processes and provide their respective theoretical analysis. We also improve upon existing algorithms for RDP learning, and propose a modified algorithm using Count-Min-Sketch with reduced memory complexity. We define a hierarchy of language families and introduce a language-restricted approach, removing the dependency on $\PrefL$-distinguishability parameters and compare the performance of our algorithms to FlexFringe, a state-of-the-art algorithm for learning probabilistic deterministic finite automata. Although CMS suffers from a large running time, the language-restricted approach offers smaller automata and optimal (or near optimal) policies, even on domains requiring long-term dependencies. Finally, as a future work, we plan to expand our approach to the online RDP learning setting. 

\bibliographystyle{abbrvnat}
\bibliography{bib}

%%%%%%%%%%%%%%%%%%%%%%%%%%%%%%%%%%%%%%%%%%%%%%%%%%%%%%%%%%%%
\newpage
\appendix

\section{Pseudocode of \textsc{AdaCT-H}}
\label{app:adacth}

\begin{function}
{\small
	%\caption{AdaCT--H($\Dataset$, $\delta$)}
	\SetKwFunction{function}{function}
	\SetKwProg{myalg}{}{}{}
  	\myalg{\function{\textsc{AdaCT--H}}}{
	%\label{alg:adact}  Use algorithm name \FAdaCT
	%
	\KwIn{Dataset $\Dataset$ containing $N$ traces in $\Episodes_H$, failure probability $0<\delta<1$}
	\KwOut{Set $\Qs$ of RDP states, transition function $\TransitionFn:\Qs \times \Actions\,\Observations \to \Qs$}
	\BlankLine
	$\Qs_0 \gets \{q_0\}$, $\Multiset(q_0) \gets \Dataset$ \tcp*{initial state}
	\For{$t = 0, \dots, H$}{
		$\Qs_{\mathsf{c},t+1} \gets \{qao \Given q\in\Qs_t, ao\in\Actions\,\Observations \}$ \tcp*{get candidate states}
		\lForEach{$qao\in\Qs_{\mathsf{c},t+1}$}{$\Multiset(qao) \gets \{e_{t+1:H} \Given aroe_{t+1:H}\in\Multiset(q)\}$ \hspace*{.25cm} \tcp*[h]{compute suffixes}}
		$q_{\mathsf{m}}a_{\mathsf{m}}o_{\mathsf{m}} \gets \arg\max_{qao\in\Qs_{\mathsf{c},t+1}} |\Multiset(qao)|$ \tcp*{most common candidate}
		$\Qs_{t+1} \gets \{q_{\mathsf{m}}a_{\mathsf{m}}o_{\mathsf{m}}\}$, $\TransitionFn(q_{\mathsf{m}},a_{\mathsf{m}}o_{\mathsf{m}})=q_{\mathsf{m}}a_{\mathsf{m}}o_{\mathsf{m}}$ \tcp*{promote candidate}
		$\Qs_{\mathsf{c},t+1} \gets \Qs_{\mathsf{c},t+1} \setminus \{q_{\mathsf{m}}a_{\mathsf{m}}o_{\mathsf{m}}\}$ \tcp*{remove from candidate states}
		\For{$qao\in\Qs_{\mathsf{c},t+1}$}{
			$\DSimilar \gets \{q'\in\Qs_{t+1} \Given $ not $\textsc{TestDistinct}(t, \Multiset(qao), \Multiset(q'), \delta)\}$ \hspace*{.01cm} \tcp*[h]{confidence test}
			
			\lIf{$\DSimilar = \emptyset$}{
				$\Qs_{t+1} \gets \Qs_{t+1}\cup\{qao\}$,
				$\TransitionFn(q,ao)=qao$ \hspace*{0.5cm} \tcp*[h]{promote candidate}
			}
   
			\lElse{
				$q' \gets $ element in $\DSimilar$,
				$\TransitionFn(q,ao)=q'$,
				$\Multiset(q') \gets \Multiset(q') \cup \Multiset(qao)$ \hspace*{-0.02cm} \tcp*[h]{merge states}
			}
		}
	}
	\Return{$\Qs_0\cup\cdots\cup\Qs_{H+1}$, $\TransitionFn$}
	}
	\setcounter{AlgoLine}{0}
	\function{\textsc{TestDistinct}}{\;
	\KwIn{Time $t$, two multisets $\Multiset_1$ and $\Multiset_2$ of traces in $(\Actions\Observations\times\Rewards)^{H-t+1}$, failure probability $0<\delta<1$}
	\KwOut{True if $\Multiset_1$ and $\Multiset_2$ are regarded as distinct, False otherwise}
	\BlankLine
	\Return $\PrefL(\Multiset_1,\Multiset_2)\geq\sqrt{2\log(8(ARO)^{H-t}/\delta)/\min(|\Multiset_1|,|\Multiset_2|)}$
	}
}
\end{function}

\section{Pseudometrics}
\label{app:pseudometric}

We provide the definition of pseudometric spaces.

\begin{definition}[Pseudometrics]
    Given a set $\mathcal K$ and a non-negative function $d: \cX \times \cX \to \Reals$, we call $(\cX, d)$ a pseudometric space with pseudometric $d$, if for all $x, y, z \in \cX$:
    \begin{itemize}
        \item[(i)] $d(x, x) = 0$,
        \item[(ii)] $d(x, y) = d(y, x)$,
        \item[(iii)] $d(x, y) + d(y, z) \geq d(x, z)$.
    \end{itemize}
\end{definition}

If it also holds that $d(x, y) = 0$ iff $x = y$, then $d$ is a metric. 

\section{Proof of theorems}
\label{app:thms}

In this appendix we prove Theorems~\ref{thm:cms} and~\ref{thm:lang} using a series of lemmas. We first describe how to implement \textsc{AdaCT-H} using CMS.

Given a finite set $\cZ$ and a probability distribution $q\in\Delta(\cZ)$, let $\widehat{q}\in\Delta(\cZ)$ be an empirical estimate of $q$ computed using $n$ samples.
CMS can store an estimate $\widetilde{q}$ of the empirical distribution $\widehat{q}$.
In this setting, the vector $v$ contains the empirical counts of each element of $\cZ$, which implies $\norm{v}_1=n$ and $\widehat{q}(z_i)=v_i/n$ for each $z_i\in\cZ$.
The following lemma shows how to bound the error between $\widetilde{q}$ and $\widehat{q}$.

\begin{lemma}
\label{lemma:cms}
Given a finite set $\cZ$, a probability distribution $q\in\Delta(\cZ)$, and an empirical estimate $\widehat{q}\in\Delta(\cZ)$ of $q$ obtained using $n$ samples, let $\widetilde{q}$ be the estimate of $\widehat{q}$ output by CMS with parameters $\delta_c$ and $\varepsilon$. With probability at least $1-|\cZ|\delta_c$ it holds that $\norm{\widetilde{q}-\widehat{q}}_\infty \leq \epsilon$. 
\end{lemma}

\begin{proof}
\citet{cormodeM05cms} show that for a point query that returns an approximation $\widetilde{v}_i$ of $v_i$, with probability at least $1-\delta_c$ it holds that
\[
\widetilde{v}_i \leq v_i + \varepsilon \norm{v}_1.
\]
In our case, the estimated probability of an element $z_i\in\cZ$ equals $\widetilde{q}(z_i)=\widetilde{v}_i/n$, where $\widetilde{v}_i$ is the point query for $z_i$.
Using the result above, with probability at least $1-\delta_c$ we have
\[
\widetilde{q}(z_i) = \frac {\widetilde{v}_i} n \leq \frac {v_i} n + \frac {\varepsilon \norm{v}_1} n = \widehat{q}(z_i) + \varepsilon.
\]
Since CMS never underestimates a value, $\widehat{q}(z_i)\leq\widetilde{q}(z_i)$ trivially holds. Taking a union bound shows that the inequality above holds simultaneously for all $z_i\in\cZ$ with probability $1-|\cZ|\delta_c$.
\end{proof}

The following two lemmas are analogous to Lemmas 13 and 14 of \citet{cipollone2023}.
\begin{lemma}
\label{lemma:upper}
For $t\in[0,H]$, let $\Multiset_1$ and $\Multiset_2$ be multisets sampled from distributions $p_1$ and $p_2$ on $\Delta(\Episodes_{H-t})$.
Assume that we use CMS with parameters $\delta_c=\delta/8(AOR)^{H-t}$ and $\varepsilon=\sqrt{ \log(2/\delta_c) / 2|\Multiset_i| }$ to store an approximation $\widetilde{p}_i$ of the empirical estimate $\widehat{p}_i$ of $p_i$ due to $\Multiset_i$, $i\in[2]$.
If $p_1=p_2$, with probability at least $1-\delta$ the statistical test satisfies
\[
\PrefL(\widetilde{p}_1, \widetilde{p}_2) \leq \sqrt{ \frac {8\log(16(AOR)^{H-t}/\delta)} {\min(|\Multiset_1|,|\Multiset_2|)} }.
\]
\end{lemma}

\begin{proof}
Hoeffding's inequality states that for each $i\in[2]$, with probability at least $1-\delta_c$ it holds for each $e\in\Episodes_u$, $u\leq H-t$, that
\[
\abs{ \widehat{p}_i(e) - p_i(e) } \leq \sqrt{ \frac {\log(2/\delta_c)} {2|\Multiset_i|} }.
\]
We can now use Lemma~\ref{lemma:cms} to upper bound the $\PrefL$ metric as
\begin{align*}
L_\Languages(\widetilde{p}_1, \widetilde{p}_2) &= \max_{u\in[0,H-t],e\in\Episodes_u} \abs{\widetilde{p}_1(e) - \widetilde{p}_2(e)}\\
 &\leq \max_{u\in[0,H-t],e\in\Episodes_u} ( \abs{\widetilde{p}_1(e) - \widehat{p}_1(e)} + \abs{\widehat{p}_1(e) - p_1(e)} + \abs{p_1(e) - p_2(e)}\\
  & \hspace*{2.3cm} + \abs{p_2(e) - \widehat{p}_2(e)} + \abs{\widehat{p}_2(e) - \widetilde{p}_2(e)} )\\
 &\leq \sqrt{ \frac {\log(2/\delta_c)} {2|\Multiset_1|} } + \sqrt{ \frac {\log(2/\delta_c)} {2|\Multiset_1|} } + \sqrt{ \frac {\log(2/\delta_c)} {2|\Multiset_2|} } + \sqrt{ \frac {\log(2/\delta_c)} {2|\Multiset_2|} } \leq \sqrt{ \frac {8\log(2/\delta_c)} {\min(|\Multiset_1|,|\Multiset_2|)} }.
\end{align*}
Taking a union bound implies that this holds with probability at least $1-8(AOR)^{H-t}\delta_c=1-\delta$.
\end{proof}

\begin{lemma}
\label{lemma:lower}
For $t\in[0,H]$, let $\Multiset_1$ and $\Multiset_2$ be multisets sampled from distributions $p_1$ and $p_2$ on $\Delta(\Episodes_{H-t})$.
Assume that we use CMS with parameters $\delta_c=\delta/8(AOR)^{H-t}$ and $\varepsilon=\sqrt{ \log(2/\delta_c) / 2|\Multiset_i| }$ to store an approximation $\widetilde{p}_i$ of the empirical estimate $\widehat{p}_i$ of $p_i$ due to $\Multiset_i$, $i\in[2]$.
If $p_1\neq p_2$ and $\min(|\Multiset_1|,|\Multiset_2|)\geq 32 \log (2/\delta_c)/\mu_0^2$, with probability at least $1-\delta$ the statistical test satisfies
\[
\PrefL(\widetilde{p}_1, \widetilde{p}_2) \geq \sqrt{ \frac {8\log(16(AOR)^{H-t}/\delta)} {\min(|\Multiset_1|,|\Multiset_2|)} }.
\]
\end{lemma}

\begin{proof}
We can use Lemma~\ref{lemma:cms} to lower bound the $\PrefL$ metric as
\begin{align*}
L_\Languages(\widetilde{p}_1, \widetilde{p}_2) &= \max_{u\in[0,H-t],e\in\Episodes_u} \abs{\widetilde{p}_1(e) - \widetilde{p}_2(e)}\\
 &\geq \max_{u\in[0,H-t],e\in\Episodes_u} ( \abs{p_1(e) - p_2(e)} - \abs{\widetilde{p}_1(e) - \widehat{p}_1(e)} - \abs{\widehat{p}_1(e) - p_1(e)}\\
  & \hspace*{2.3cm} - \abs{p_2(e) - \widehat{p}_2(e)} - \abs{\widehat{p}_2(e) - \widetilde{p}_2(e)} )\\
 &\geq \mu_0 - \sqrt{ \frac {\log(2/\delta_c)} {2|\Multiset_1|} } - \sqrt{ \frac {\log(2/\delta_c)} {2|\Multiset_1|} } - \sqrt{ \frac {\log(2/\delta_c)} {2|\Multiset_2|} } - \sqrt{ \frac {\log(2/\delta_c)} {2|\Multiset_2|} }\\
  &\geq \mu_0 - \sqrt{ \frac {8\log(2/\delta_c)} {\min(|\Multiset_1|,|\Multiset_2|)} } \geq \mu_0 - \frac {\mu_0} 2 \geq \frac {\mu_0} 2 \geq \sqrt{ \frac {8\log(2/\delta_c)} {\min(|\Multiset_1|,|\Multiset_2|)} }.
\end{align*}
Taking a union bound implies that this holds with probability at least $1-8(AOR)^{H-t}\delta_c=1-\delta$.
\end{proof}

The remainder of the proof of Theorem~\ref{thm:cms} follows exactly the same steps as in the proof of \citet[Theorem 6]{cipollone2023}.
For completeness, we repeat the steps here.
The proof consists in choosing $N=|\Dataset|$ and $\delta$ such that the condition on $\min(|\Multiset_1|,|\Multiset_2|)$ in Lemma~\ref{lemma:lower} is true with high probability for each application of \textsc{TestDistinct}. Consider an iteration $t\in[0,H]$ of \textsc{AdaCT-H}.
For a candidate state $qao\in\Qs_{\mathsf{c},t+1}$, its associated probability is $\Behav\COccupancy_t(q,ao)$ with empirical estimate $\widehat{p}_t(qao)=|\Multiset(qao)| / N$, i.e.~the proportion of episodes in $\Dataset$ that are consistent with $qao$. We can apply an empirical Bernstein inequality to show that
\begin{align*}
\Pr \pa{ \left\lvert \widehat{p}_t(qao) - \Behav\COccupancy_t(q,ao) \right\rvert \geq \sqrt{\frac{2\widehat{p}_t(qao)\ell} N} + \frac {14\ell} {3N} = \frac { \sqrt{ 2M\ell } + 14\ell/3 } N } \leq \delta,
\end{align*}
where $M=|\Multiset(qao)|$, $\ell=\log(4/\delta)$, and $\delta$ is the failure probability of \textsc{AdaCT-H}. To obtain a bound on $M$ and $N$, assume that we can estimate $\Behav\COccupancy_t(q,ao)$ with  accuracy $\Behav\COccupancy_t(q,ao) / 2$, which yields
\begin{align}
\frac {\Behav\COccupancy_t(q,ao)} 2 &\geq \frac { \sqrt{ 2M\ell } + 14\ell/3 } N\label{eq:Nocc}\\
\widehat{p}_t(qao) &\geq \Behav\COccupancy_t(q,ao) - \frac { \sqrt{ 2M\ell } + 14\ell/3 } N \geq \Behav\COccupancy_t(q,ao) - \frac {\Behav\COccupancy_t(q,ao)} 2 = \frac {\Behav\COccupancy_t(q,ao)} 2.
\end{align}
Combining these two results, we obtain
\begin{align}
M &= N\widehat{p}_t(qao) \geq N\Behav\COccupancy_t(q,ao)/2 \geq \frac N {2N} \pa{ \sqrt{ 2M\ell } + 14\ell/3 } = \frac 1 2 \pa{ \sqrt{ 2M\ell } + 14\ell/3 }.\label{eq:NM}
\end{align}
Solving for $M$ yields $M\geq 4\ell$, which is subsumed by the bound on $M$ in Lemma~\ref{lemma:lower} since $\CDistinguish<1$.
Hence the bound on $M$ in Lemma~\ref{lemma:lower} is sufficient to ensure that we estimate $\Behav\COccupancy_t(q,ao)$ with accuracy $\Behav\COccupancy_t(q,ao) / 2$.
We can now insert the bound on $M$ from Lemma~\ref{lemma:lower} into \eqref{eq:Nocc} to obtain a bound on $N$:
\begin{align}\label{eq:N1}
N &\geq \frac {2 ( \sqrt{ 2M\ell } + 14\ell/3) } {\Behav\COccupancy_t(q,ao)} \geq \frac {2\ell} {\Behav\COccupancy_t(q,ao)} \pa{ \frac 8 \CDistinguish \sqrt{\frac {(H-t)\log(4ARO)} \ell + 1 } + \frac {14} 3 } \equiv N_1.
\end{align}
To simplify the bound, we can choose any value larger than $N_1$:
\begin{align}
N_1 &\leq \frac {2\ell} {\Behav\COccupancy_t(q,ao)} \pa{ \frac 8 \CDistinguish \sqrt{H\log(4ARO) + H\log(4ARO) } + \frac {14} {3\CDistinguish} \sqrt{H\log(4ARO)} }\nonumber\\
 &< \frac {32\ell} {\Behav\COccupancy_{\min}\,\CDistinguish} \sqrt{H\log(4ARO)} \equiv N_0,
\end{align}
where we have used $\Behav\COccupancy_t(q,ao)\geq \Behav\COccupancy_{\min}$, $\CDistinguish<1$, $\ell=\log 4 + \log(1/\delta)\geq 1$, $H\log(4ARO)\geq\log 4\geq 1$ and $8\sqrt{2} + 14/3 < \frac {32} 2$. Choosing $\delta=\delta_0/2QAO$, a union bound implies that accurately estimating $\Behav\COccupancy_t(q,ao)$ for each candidate state $qao$ and accurately estimating $p(e_{0:u}*)$ for each prefix in the multiset $\Multiset(qao)$ associated with $qao$ occurs with probability $1-2QAO\delta=1-\delta_0$, since there are at most $QAO$ candidate states. Ignoring logarithmic terms, this equals the bound in the theorem.

It remains to show that the resulting RDP is minimal. We show the result by induction. The base case is given by the set $\Qs_0$, which is clearly minimal since it only contains the initial state $q_0$. For $t\in[0,H]$, assume that the algorithm has learned a minimal RDP for sets $\Qs_0,\ldots,\Qs_t$. Let $\Qs_{t+1}$ be the set of states at layer $t+1$ of a minimal RDP. Each pair of histories that map to a state $q_{t+1}\in\Qs_{t+1}$ generate the same probability distribution over suffixes. Hence by Lemma~\ref{lemma:upper}, with high probability \textsc{TestDistinct}$(t,\Multiset(qao),\Multiset(q'a'o'),\delta)$ returns false for each pair of candidate states $qao$ and $q'a'o'$ that map to $q_{t+1}$. Consequently, the algorithm merges $qao$ and $q'a'o'$. On the other hand, by assumption, each pair of histories that map to different states of $\Qs_{t+1}$ have $\PrefL$-distinguishability $\CDistinguish$. Hence by Lemma~\ref{lemma:lower}, with high probability \textsc{TestDistinct}$(t,\Multiset(qao),\Multiset(q'a'o'),\delta)$ returns true for each pair of candidate states $qao$ and $q'a'o'$ that map to different states in $\Qs_{t+1}$. Consequently, the algorithm does not merge $qao$ and $q'a'o'$. It follows that with high probability, \FAdaCT will generate exactly the set $\Qs_{t+1}$, which is that of a minimal RDP.

The proof of Theorem~\ref{thm:lang} is achieved by proving two very similar lemmas.

\begin{lemma}
For $t\in[0,H]$, let $\Multiset_1$ and $\Multiset_2$ be multisets sampled from distributions $p_1$ and $p_2$ on $\Delta(\Episodes_{H-t})$, and let $\widehat{p}_1$ and $\widehat{p}_2$ be empirical estimates of $p_1$ and $p_2$ due to $\Multiset_1$ and $\Multiset_2$, respectively.
If $p_1=p_2$, with probability at least $1-\delta$ the statistical test satisfies
\[
L_\Languages(\widehat{p}_1, \widehat{p}_2) \leq \sqrt{ \frac {2\log(4|\Languages|/\delta)} {\min(|\Multiset_1|,|\Multiset_2|)} }.
\]
\end{lemma}

\begin{proof}
Let $\delta_\ell=\delta/2|\Languages|$.
Hoeffding's inequality states that for each $\Language\in\Languages$ and each $i\in\{1,2\}$, with probability at least $1-\delta_\ell$ it holds that
\[
\abs{ \widehat{p}_i(\Language) - p_i(\Language) } \leq \sqrt{ \frac {\log(2/\delta_\ell)} {2|\Multiset_i|} }.
\]
We can now upper bound the language metric as
\begin{align*}
L_\Languages(\widehat{p}_1, \widehat{p}_2) &= \max_{\Language \in \Languages} \abs{\widehat{p}_1(\Language) - \widehat{p}_2(\Language)}\\
 &\leq \max_{\Language \in \Languages} ( \abs{\widehat{p}_1(\Language) - p_1(\Language)} + \abs{p_1(\Language) - p_2(\Language)} + \abs{p_2(\Language) - \widehat{p}_2(\Language)} )\\
 &\leq \sqrt{ \frac {\log(2/\delta_\ell)} {2|\Multiset_1|} } + 0 + \sqrt{ \frac {\log(2/\delta_\ell)} {2|\Multiset_2|} } \leq \sqrt{ \frac {2\log(2/\delta_\ell)} {\min(|\Multiset_1|,|\Multiset_2|)} }.
\end{align*}
Taking a union bound implies that this bound holds with probability at least $1-2|\Languages|\delta_\ell=1-\delta$.
\end{proof}

\begin{lemma}
For $t\in[0,H]$, let $\Multiset_1$ and $\Multiset_2$ be multisets sampled from distributions $p_1$ and $p_2$ on $\Delta(\Episodes_{H-t})$, and let $\widehat{p}_1$ and $\widehat{p}_2$ be empirical estimates of $p_1$ and $p_2$ due to $\Multiset_1$ and $\Multiset_2$, respectively.
If $p_1\neq p_2$ and $\min(|\Multiset_1|,|\Multiset_2|)\geq 8\log(4|\Languages|/\delta)/\mu_0^2$, with probability at least $1-\delta$ the statistical test satisfies
\[
L_\Languages(\widehat{p}_1, \widehat{p}_2) \geq \sqrt{ \frac {2\log(4|\Languages|/\delta)} {\min(|\Multiset_1|,|\Multiset_2|)} }.
\]
\end{lemma}

\begin{proof}
We can lower bound the language metric as
\begin{align*}
L_\Languages(\widehat{p}_1, \widehat{p}_2) &= \max_{\Language \in \Languages} \abs{\widehat{p}_1(\Language) - \widehat{p}_2(\Language)}\\
 &\geq \max_{\Language \in \Languages} ( \abs{p_1(\Language) - p_2(\Language)} - \abs{\widehat{p}_1(\Language) - p_1(\Language)} - \abs{p_2(\Language) - \widehat{p}_2(\Language)} )\\
 &\geq \mu_0 - \sqrt{ \frac {\log(2/\delta_\ell)} {2|\Multiset_1|} } - \sqrt{ \frac {\log(2/\delta_\ell)} {2|\Multiset_2|} } \geq \mu_0 - \sqrt{ \frac {2\log(4|\Languages|/\delta)} {\min(|\Multiset_1|,|\Multiset_2|)} }\\
 &\geq \mu_0 - \frac {\mu_2} 2 \geq \frac {\mu_0} 2 \geq \sqrt{ \frac {2\log(4|\Languages|/\delta)} {\min(|\Multiset_1|,|\Multiset_2|)} }.
\end{align*}
Taking a union bound implies that this bound holds with probability at least $1-2|\Languages|\delta_\ell=1-\delta$.
\end{proof}
The remainder of the proof of Theorem~\ref{thm:lang} is analogous to that of Theorem~\ref{thm:cms}.
However, we no longer get a term $\sqrt{H}$ from $\sqrt{\log((AOR)^H)}$ unless the number of languages in $\Languages$ is exponential in $H$.

\section{Details of the experiments}
\label{app:exps}

In this appendix we describe each domain in detail and include examples of RDPs learned by \textsc{AdaCT-H}.

\subsection{Corridor}
This RDP example was introduced in  \citet{ronca_2021_EfficientPACa}. The environment consists of a $2\times m$ grid, with only two actions $a_0$ and $a_1$ which moves the agent to states $(0,i+1)$ and $(1,i+1)$ respectively from state $(\cdot, i)$. The goal of the agent is to avoid an enemy which is present in position $(0,i)$ with probability $p_i^0$, and at $(1,i)$ with probability $p_i^1$. The agent receives a reward of $+1$ for avoiding the enemy at a particular column, and the probabilities $p_i^0$ and $p_i^i$ are switched every time it encounters the enemy. When the agent reaches the last column, its position is reset to the first column. The observation space is given by $(i,j,e)$, where $i,j$ is the cell position of the agent and $e \in \{\textit{enemy}, \textit{clear} \}$ denotes the presence of the guard in the current cell. Fig \ref{fig:corridor} shows the minimal automaton obtained by all three algorithms for $H=5$.
\begin{figure}[h!]
  \centering
  \includegraphics[scale=0.28]{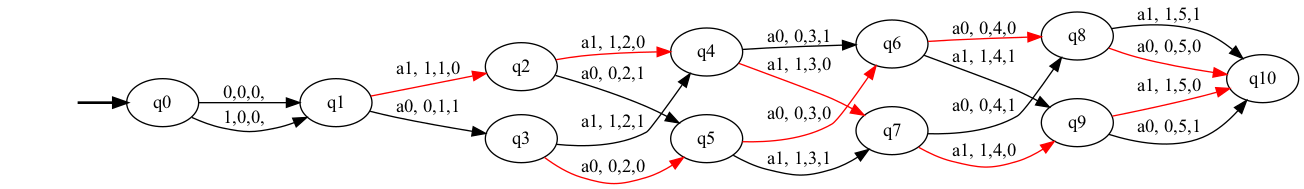}
  \caption{Automaton obtained from the corridor environment. The edges are labelled as [$\textit{action}$, $\textit{observation}$, $\textit{enemy}$].}
  \label{fig:corridor}
\end{figure}

\subsection{T-maze}
The T-maze environment was introduced by Bakker~\cite{bakker2001rlandlstm} to capture long term dependencies with RL-LSTMs. As shown in Figure \ref{fig:tmaze}, at the initial position $S$, the agent receives an observation $X$, depending on the position of the goal state $G$ in the last column. The agent can take four actions, $\textit{North}$, $\textit{South}$, $\textit{East}$ and $\textit{West}$. The agent receives a reward of $+4$ on taking the correct action at the T-junction, and $-1$ otherwise, terminating the episode. The agent also recieves a $-1$ reward for standing still. At the initial state the agent receives observation $011$ or $110$, $101$ throughout the corridor and $010$ at the T-junction. Fig \ref{fig:tmaze} shows the optimal automaton obtained when the available actions in the corridor are restricted to only $East$ (the automaton obtained without this restriction is shown in~Fig.~\ref{fig:tmaze_unres}). Table~\ref{table:1} shows our results with the unrestricted action space. Both our approaches find the optimal policy in this case, unlike Flex-Fringe which fails to capture this long term dependency. 
\begin{figure}[h!]
  \centering
  \includegraphics[scale=0.28]{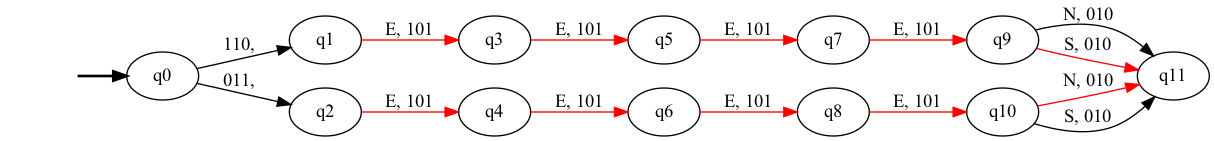}
  \caption{Automaton obtained from T-maze environment with restricted actions. The edges are labelled as [$\textit{action}$, $\textit{observation}$, $\textit{reward}$].}
  \label{fig:tmaze-restricted}
\end{figure}

% \subsection{Mini-hall}

\begin{figure}
     \centering
     \begin{subfigure}[b]{0.48\textwidth}
         \centering
         \includegraphics[scale=0.15]{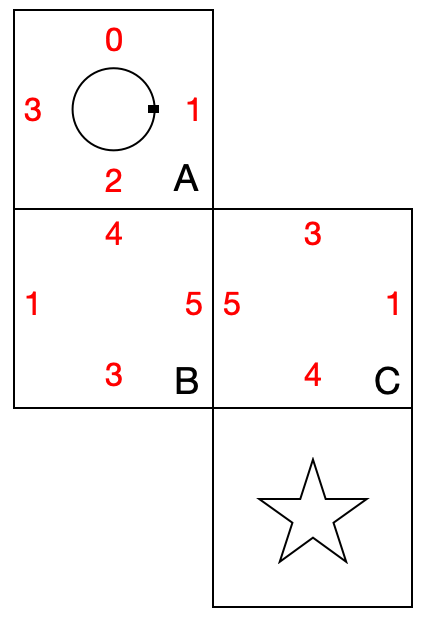}
         \caption{Mini-hall \cite{littman} environment.}
         \label{fig:minihall}
     \end{subfigure}
     \hfill
     \begin{subfigure}[b]{0.48\textwidth}
         \centering
         \includegraphics[scale=0.2]{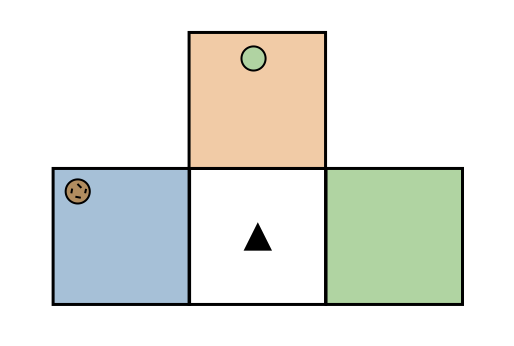}
        \caption{Simplified cookie domain.}
        \label{fig:cookie}
     \end{subfigure}
    \caption{Environments.}
    \label{fig:environments}
\end{figure}

% The mini-hall environment~\cite{littman} shown in Fig \ref{fig:minihall} has $12$ states, $4$ orientations in $3$ rooms, a goal state given by a star associated with a reward of $1$. The agent can take $3$ actions, $\textit{forward}$, $\textit{rotate}$ $\textit{left}$, $\textit{rotate}$ $\textit{right}$ with deterministic transition and observations. After reaching the goal state, any actions resets the agent position to any non-goal state with equal probability. There are $8$ possible observations (relative positions of the walls, and goal state). 
\begin{figure}[!h]
  \centering
  \includegraphics[scale=0.25]{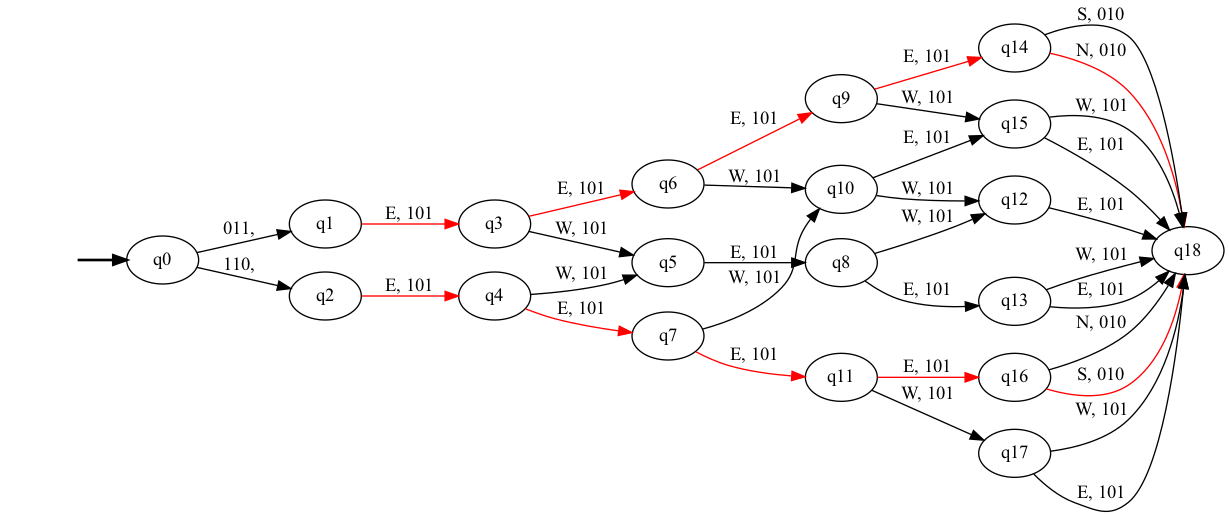}
  \caption{Automaton obtained from T-maze (partially restricted action space), restricted language.}
  \label{fig:tmaze_unres}
\end{figure}

\subsection{Cookie domain}

We modify the original \emph{cookie domain} as described in Icarte et al.~\cite{toroicarte2019learning}, to a simpler domain consisting of $4$ rooms, $\textit{blue}$, $\textit{white}$, $\textit{green}$ and $\textit{red}$ as shown in Fig.~\ref{fig:cookie}. If the agent presses the button in room $\textit{red}$, a cookie appears in room $\textit{blue}$ or $\textit{green}$ with equal probability. The agent can move $\textit{left}$, $\textit{right}$, $\textit{up}$ or $\textit{down}$, can $\textit{press}$ the button in room $\textit{red}$, and $\textit{eat}$ the cookie to receive a reward $1$, and then it may press the button again. There are $6$ possible observations ($4$ for each room, and $2$ for observing the $\textit{cookie}$ in the two rooms). We use the set $\Languages_{1,1,1}$  for distinguishability in the restricted language case. Our restricted language approach here finds the optimal policy and the smallest state space.

\subsection{Cheese maze}
Cheese maze~\cite{McCallum1996ReinforcementLW} consists of $10$ states, and $6$ observations, and $4$ actions. After reaching the goal state, the agent receives a reward of $+1$, and the position of the agent is reintialised to one of the non-goal states with equal probability. Our restricted language approach uses the set $\Languages_{1,1,1}$. For a horizon of $6$, the results for the restricted language and Flex-Fringe are comparable, however upon further increasing the horizon, Flex-Fringe outperforms \textsc{AdaCT-H} by learning cyclic RDPs which is not possible in our approach. 

\subsection{Mini-hall}
The mini-hall environment~\cite{littman} shown in Fig \ref{fig:minihall} has $12$ states, $4$ orientations in $3$ rooms, a goal state given by a star associated with a reward of $+1$, $6$ observation and $3$ actions, and the position of the agent is reset after the goal is reached. This setting is much more complex than the others because  $12$ states are mapped into $6$ observations; for example, starting from observation $3$, $3$ actions are required under the optimal policy to solve the problem if the starting underlying state was in Room B or C. We use the set $\Languages_{1,1,1}$ in our restricted language approach for distinguishability. Although we get a much larger state space, our algorithm gets closer to the optimal policy. However our CMS approach is not efficient in this case and exceeds the alloted time budget of $1800$ seconds, as it requires to iterate over the entire length of trajectories.

\end{document}